\documentclass[a4paper,11pt]{article}
\usepackage[utf8]{inputenc}
\usepackage[margin=1in]{geometry}
\usepackage[colorlinks,linkcolor=blue,citecolor=blue]{hyperref}

\usepackage{float,framed}

\usepackage[round]{natbib}
\usepackage{enumitem}
\usepackage{xspace}

\usepackage[colorlinks,linkcolor=blue,citecolor=blue]{hyperref}
\usepackage{amsmath,amsfonts,amssymb}
\usepackage{mathtools}
\usepackage{amsthm} 
\usepackage{relsize}

\usepackage{bm}

\usepackage[capitalise,nameinlink]{cleveref}
\usepackage{xcolor}
\usepackage{dsfont}

\newcommand{\ignore}[1]{}
\newcommand{\tk}[1]{\noindent{\textcolor{cyan}{\{{\bf TK:} \em #1\}}}}

\newcommand{\wrapalgo}[2][0.9\linewidth]
{%
\begin{center}\setlength{\fboxsep}{5pt}\fbox{\begin{minipage}{#1}
#2
\end{minipage}}\end{center}
\vspace{-0.25cm}
}

\theoremstyle{definition}
\newtheorem{theorem}{Theorem}
\newtheorem{lemma}[theorem]{Lemma}

\newtheorem*{theorem*}{Theorem}
\newtheorem*{lemma*}{Lemma}
\newtheorem*{corollary*}{Corollary}
\newtheorem*{proposition*}{Proposition}
\newtheorem*{claim*}{Claim}
\newtheorem*{fact*}{Fact}
\newtheorem*{observation*}{Observation}

\theoremstyle{definition}

\newtheorem*{definition*}{Definition}
\newtheorem*{remark*}{Remark}
\newtheorem*{example*}{Example}

 \theoremstyle{plain}
\newtheorem*{theoremaux}{\theoremauxref}
\gdef\theoremauxref{1}

\DeclareMathAlphabet{\mathbfsf}{\encodingdefault}{\sfdefault}{bx}{n}

\DeclareMathOperator*{\argmin}{arg\!\min}

\DeclareMathOperator*{\trace}{Tr}
\DeclareMathOperator*{\diag}{diag}

\newcommand{\lr}[1]{\mathopen{}\left(#1\right)}
\newcommand{\Lr}[1]{\mathopen{}\big(#1\big)\mathclose{}}
\newcommand{\LR}[1]{\mathopen{}\Big(#1\Big)\mathclose{}}

\newcommand{\norm}[1]{\|#1\|}

\newcommand{\set}[1]{\{#1\}}
\newcommand{\lrset}[1]{\mathopen{}\left\{#1\right\}\mathclose{}}
\newcommand{\Lrset}[1]{\mathopen{}\big\{#1\big\}\mathclose{}}

\newcommand{\wt}[1]{\smash{\widetilde{#1}}}

\newcommand{\tsum}{\sum\nolimits}
\newcommand{\tr}{^{\mkern-1.5mu\mathsf{T}}}

\newcommand{\st}{\star}

\newcommand{\reals}{\mathbb{R}}
\newcommand{\eps}{\epsilon}

\newcommand{\half}{\frac{1}{2}}
\newcommand{\thalf}{\tfrac{1}{2}}

\let\oldtfrac\tfrac
\renewcommand{\tfrac}[2]{\smash{\oldtfrac{#1}{#2}}}

\let\nablaold\nabla
\renewcommand{\nabla}{\nablaold\mkern-1mu}

\floatstyle{plain}
\newfloat{algorithm}{h}{lop}
\floatname{algorithm}{Algorithm}

\title{A Unified Approach to Adaptive Regularization \\
	in Online and Stochastic Optimization}
\author{Vineet Gupta, Tomer Koren, and Yoram Singer%
\thanks{Emails: \texttt{vineet@google.com}, \texttt{tkoren@google.com}, \texttt{singer@google.com}.}
\vspace{1.5ex}\\
Google Brain}

\newcommand{\cX}{\mathcal{X}}

\newcommand{\cH}{\mathcal{H}}
\newcommand{\cS}{\mathcal{S}}

\newcommand{\Id}{I}

\newcommand{\proj}{\Pi}
\newcommand{\cone}{\cS_{+}}
\newcommand{\coneD}{\cS_{+}^\mathsf{di}}
\newcommand{\coneI}{\cS_{+}^\mathsf{id}}

\newcommand{\PhiAG}{\Phi_\textsf{AG}}
\newcommand{\PhiONS}{\Phi_\textsf{ONS}}

\newcommand{\ar}{AdaReg\xspace}
\newcommand{\diam}{\mathsf{b}}
\newcommand{\progr}{\Delta}
\renewcommand{\det}[1]{\left|#1\right|}

\crefname{step}{step}{steps}

\begin{document}
\maketitle

\begin{abstract}
We describe a framework for deriving and analyzing online optimization
algorithms that incorporate adaptive, data-dependent regularization, also
termed preconditioning. Such algorithms have been proven useful in
stochastic optimization by reshaping the gradients according to the geometry of
the data. Our framework captures and unifies much of the existing literature
on adaptive online methods, including the AdaGrad and Online Newton Step
algorithms as well as their diagonal versions. As a result, we obtain new
convergence proofs for these algorithms that are substantially simpler than
previous analyses. Our framework also exposes the rationale for the different
preconditioned updates used in common stochastic optimization methods.
\end{abstract}

\section{Introduction} \label{sec:intro}
In Online Convex Optimization~\citep{zinkevich2003online, shalev2012online,
hazan2016introduction} a learner makes predictions in the form of a vector
belonging to a convex domain $\cX \subseteq \reals^d$ for $T$ rounds. After
predicting $x_t \in \cX$ on round $t$, a convex function $f_t : \cX \mapsto
\reals$ is revealed to the learner, potentially in an adversarial or
adaptive way, based on the learner's past predictions. The learner then
endures a loss $f_t(x_t)$ and also receives its gradient $\nabla f_t(x_t)$
as feedback.%
\footnote{Our analysis is applicable with minor changes to non-differentiable
convex functions with subgradients as feedback.}

The goal of the learner is to achieve low cumulative loss, coined
regret, with respect to any fixed vector in the $\cX$. Formally, the learner
attempts to cap the quantity
\begin{align*}
R_T = \sum_{t=1}^T f_t(x_t) - \min_{x \in \cX} \sum_{t=1}^T f_t(x) ~ .
\end{align*}
%
Online Convex Optimization has been proven useful in the context of
stochastic convex optimization, and numerous algorithms in this domain
can be seen and analyzed as online optimization methods; we again
refer to \citep{hazan2016introduction} for a thorough survey of many
of these algorithms.
Any online algorithm achieving a sublinear regret $R_T = o(T)$ can be
readily converted to a stochastic convex optimization
algorithm with convergence rate $O(R_T/T)$, using a standard technique
called online-to-batch conversion~\citep{cesa2004generalization}.

The online approach is particularly effective for the analysis of adaptive
optimization methods, namely, algorithms that change the nature of their
update rule on-the-fly so as to adapt to the geometry of the observed
data (i.e., perceived gradients).  The update rule of such algorithms often
takes the form $x_{t+1} \gets x_t - H_t g_t$ where $g_t$ is a (possibly
stochastic) gradient of the function $f_t$ evaluated at $x_t$, and $H_t$ is a
\emph{regularization matrix}, or a \emph{preconditioner}, used to skew the
gradient step in a desirable way. Importantly, the matrix $H_t$ may be
chosen in an adaptive way based on past gradients, and might even depend on
the gradient $g_t$ of the same step. The online optimization apparatus, in
which the objective functions $f_t$ may vary almost arbitrarily, is very
effective in dealing with these intricate dependencies. For a recent survey on
adaptive methods in online learning and their analysis techniques
see~\citep{mcmahan2014analysis}.

One of the well-known adaptive online algorithms is AdaGrad
\citep{duchi2011adaptive} which is commonly used in machine learning for
training sparse linear models. AdaGrad also became popular for training deep
neural networks. Intuitively, AdaGrad employs an adaptive regularization for
maintaining a step-size on a per-coordinate basis, and can thus perform
aggressive updates on informative yet rarely seen features (a similar
approach was also taken by~\citealp{mcmahan2010adaptive}). Another adaptive
algorithm known in the online learning literature is the Online Newton Step
(ONS) algorithm \citep{hazan2007logarithmic}. ONS incorporates an adaptive
regularization technique for exploiting directional (non-isotropic) curvature
of the objective function.  While these adaptive regularization algorithms appear similar to each other, their derivation and analysis are disparate and
technically involved. Furthermore, it is often difficult to gain insights into
the specific choices of the matrices used for regularization and what role do
they play in the analysis of the resulting algorithms.

In this paper, we present a general framework from which adaptive algorithms
such AdaGrad and ONS can be derived using a streamlined scheme. Our framework
is parameterized by a \emph{potential function} $\Phi$. Different choices of
$\Phi$ give rise to concrete adaptive algorithms. Morally, after choosing a
potential $\Phi$, the algorithm computes its regularization matrix, a
preconditioner, for iterate $t$ by solving a minimization of the form,
\begin{align}\label{eq:Ht}
\min_{H \succ 0} \lrset{ \sum_{s=1}^t \norm{g_s}_H^2 + \Phi(H) } ~ .
\end{align}
Thus the algorithm strikes a balance between the potential of $H$,
$\Phi(H)$, and the quality of $H$ as a regularizer for controlling the norms
of the gradients with respect to the observations thus far. Not only does this balance
give a natural interpretation of the regularization used by common adaptive
algorithms, it also makes their analysis rather simple: an adaptive
regularization algorithm can be viewed as a
follow-the-leader (FTL) algorithm that operates over the class of positive
definite matrices. We can thus analyze adaptive regularization
methods using simple and well established FTL analyses. 

Solving the minimization above over positive definite matrices is, in general,
a non-trivial task. However, in certain cases we can obtain a closed form
solution that gives rise to efficient algorithms.  For instance, to obtain
AdaGrad we pick the potential $\Phi(H) = \trace(H^{-1})$ and solve the
minimization via elementary differentiation, which leads to regularizers of
the form $H_t = (\sum_{s=1}^t g_s g_s\tr)^{-1/2}$. To
obtain ONS we pick $\Phi(H) = -\log\det{H}$ which yields
$H_t = \smash{(\sum_{s=1}^t g_s g_s\tr)^{-1}}$,
which constitutes the ONS update.

For both AdaGrad and ONS, we also derive diagonal versions of the
algorithms by constraining the minimization to diagonal
positive definite matrices. We also show that by further constraining the
minimization to positive multiples of the identity matrix, one can recover
familiar matrix-free (scalar) online algorithms that adaptively tune their
step-size parameter according to observed gradients. As in the case of full
matrices, the resulting minimization over matrices can be solved in closed
form and the analyses follow seamlessly from the choice of the potential.
Last we would like to note that the analysis applies to the mirror-descent family of algorithms; nevertheless, our approach can also be used to analyze dual-averaging-type algorithms, also referred to as follow-the-regularized-leader algorithms.

\paragraph{Notation.}
We denote by $\cone$ the positive definite cone, i.e. the set of all
$d \times d$ positive definite matrices. We use $\diag(A)$ to denote the
diagonal matrix whose diagonal coincides the diagonal elements of $A$ and its
off-diagonal elements are $0$.  The trace of the matrix $A$ is denoted as
$\trace(A)$.  The element-wise inner-product of matrices $A$ and $B$ is
denoted as $A \bullet B = \trace(A\tr B)$.

The spectral norm of a matrix $A$ is denoted $\|A\|_2 = \max \|A x\|/\|x\|$
where $x\neq 0$.
We denote by $\|x\|_H = \sqrt{x\tr H x}$ the norm of $x\in\reals^d$ with
respect to a positive definite matrix $H \in \cone$.  The dual norm of
$\|\cdot\|_H$ is denoted $\|\cdot\|_{H^*}$ and is equal to
$\sqrt{x\tr H^{-1} x}$. We denote by
$$\Pi_{\cX}^{H}\big(x\big) = \argmin_{x'\in\cX} \|x'-x\|_{H}$$
the projection of $x$ onto a bounded convex set $\cX$ with respect to the norm
induced by $H \in \cone$.  When $H = I$, the identity matrix, we omit the
superscript and simply use $\Pi_\cX$ to denote the typical Euclidean
projection operator.

Given a symmetric $d \times d$ matrix $A$ and a function
$\phi : \reals \mapsto \reals$, we define $\phi(A)$ as the
$d \times d$ matrix obtained by applying $\phi$ to the eigenvalues of $A$.
Formally, let us rewrite $A$ using its spectral decomposition,
$\smash{\sum_{i=1}^d \lambda_i u_i u_i\tr}$ where $\lambda_i,u_i$ are $A$'s $i$'th
eigenvalue and eigenvector respectively.  Then, we define
$\phi(A) = \smash{\sum_{i=1}^d \phi(\lambda_i) u_i u_i\tr}$. The function $\phi$ is
said to be \emph{operator monotone} if $A \succeq B \succeq 0$ implies that
$\phi(A) \succeq \phi(B)$. A classic result in matrix theory used in our
analysis is the L\"owner-Heinz Theorem (see, for instance
Theorem 2.6 in \citealp{carlen2010trace}), which in particular asserts that the function
$\phi(x) = x^\alpha$ is operator monotone for any $\alpha\in[0,1]$.
(Interestingly, it is not the case for $\alpha>1$.)
We also use an elementary identity from matrix calculus to compute derivatives of matrix traces:
$ \nabla_A \trace(\phi(A)) = \phi'(A)$.

\section{Unified Adaptive Regularization}
In this section we describe and analyze the meta-algorithm for Adaptive
Regularization (\ar). The pseudocode of the algorithm is given
in~\cref{alg:universal}. \ar constructs a succession of matrices $H_t$, each
multiplies its instantaneous gradient $g_t$. The matrices act as
pre-conditioners which reshape the gradient-based directions. In order to
construct the pre-conditioners \ar is provided with a potential function $\Phi
: \cH \mapsto \reals$ over a subset $\cH$ of the positive definite matrices.
On each round, $\Phi$ casts a trade-off involving two terms. The first term
promotes pre-conditioners which are inversely proportional to the accumulated
outer products of gradients, namely,
\begin{align} \label{Gt:eqn}
G_t &= G_0 + \sum_{s=1}^t g_s g_s\tr ~ .
\end{align}
The second term ``pulls'' back towards typically the zero matrix and is
facilitated by $\Phi$.  We define the initial regularizer
$H_0 = \min_{H \in \cH} \lrset{ G_0 \bullet H + \Phi(H)}$.

\begin{algorithm}[t]
\wrapalgo[0.90\textwidth]{
{\bf Parameters:}
	$\cH \subset \cone$,
	potential $\Phi : \cH \mapsto \reals$,
	vector $x_1 \in \cX$,
	and matrix $G_0 \in \cone$ \\
For $t=1,2,\ldots,T$:
\begin{enumerate}[nosep,label=(\arabic*)]
	\item Output: $x_t$; Receive: $f_t$
	\label[step]{alg-step1} \smallskip
	\item Compute: $g_t = \nabla f_t(x_t)$ and $G_t = G_{t-1} + g_t g_t\tr$
	\label[step]{alg-step2} \smallskip
	\item Calculate:
	$ H_t = \argmin_{H \in \cH} \Lrset{ G_t \bullet H + \Phi(H) } $
	\label[step]{alg-step3} \smallskip
	\item Update:
	$
	x_{t+1}
	=
	\smash{\proj_{\cX}^{H_t^*}\Lr{ x_t - H_t g_t }}
	$
	\label[step]{alg-step4}
\end{enumerate}
}
\caption{Adaptive regularization meta-algorithm.}
\label{alg:universal}
\end{algorithm}

We now state the main regret bound we prove for \cref{alg:universal}, from
which all the results in this paper are derived.
\begin{theorem} \label{thm:universal}
For any $x^\star \in \cX$ it holds that
\begin{align} \label{eq:universal}
\sum_{t=1}^T f_t(x_t) - \sum_{t=1}^T f_t(x^\star)
\le
\half \min_{H \in \cH} \lrset{ G_T \bullet H + \Phi(H) - \Phi(H_0)}
+
\half \sum_{t=1}^T \progr_t(x^\st)
,
\end{align}
where $\progr_t(x^\st) = \norm{x_t-x^\star}_{H_t^*}^2 - \norm{x_{t+1}-x^\star}_{H_t^*}^2$.
\end{theorem}
\noindent Note that
$$G_T \bullet H = \sum_{t=1}^T \norm{g_t}_H^2 + G_0 \bullet H ~.$$
That is, the regret of the algorithm is controlled by the magnitude of the
gradients measured by a norm $\norm{\cdot}_H$ which is, in some sense, the
best possible in hindsight: it is the one that minimizes the sum of the
gradients' norms plus a regularization term.  The regularization term, that
stems from the choice of the potential function $\Phi$, facilitates
an explicit trade-off in the resulting regret bound between minimizing the
gradients' norms with respect to $\norm{\cdot}_H$ and controlling the
magnitude of $\Phi(H) - \Phi(H_0)$. 
The second summation term in the regret bound measures the stability of the algorithm in choosing its regularization matrices: an algorithm that changes the matrices $H_t$ frequently and abruptly is thus unlikely to perform well.

By definition, the minimization on the right-hand side of
\cref{eq:universal} is attained at $H_T$, thus the bound of
\cref{thm:universal} can be rewritten as
\begin{align} \label{reregret:eqn}
\sum_{t=1}^T f_t(x_t) - \sum_{t=1}^T f_t(x^\star)
	& \le \half \Lr{ G_T \bullet H_T +  \Phi(H_T) - \Phi(H_0)} +
	\half \sum_{t=1}^T \progr_t(x^\st) ~ .
\end{align}
To prove \cref{thm:universal}, we rely on two standard tools in online
optimization. The first is the Follow-the-Leader / Be-the-Leader
(FTL-BTL) lemma.

\begin{lemma}[FTL-BTL Lemma, \citealp{kalai2005efficient}]
\label{lem:ftl-btl}
Let $\psi_0,\ldots,\psi_T : \cX \mapsto \reals$ be an arbitrary sequence of
functions defined over a domain $\cX$.  For $t \ge 0$, let
$x_t \in \argmin_{x \in \cX} \sum_{s=0}^t \psi_s(x)$,
then,
\begin{align*}
\sum_{t=1}^T \psi_t(x_t)
\le
\sum_{t=1}^T \psi_t(x_T) + (\psi_0(x_T) - \psi_0(x_0))
~ .
\end{align*}
(The term $\psi_0(\cdot)$ term is often used as regularization.)
\end{lemma}

The second tool is a standard bound for the Online Mirror Descent (OMD)
algorithm, that allows for a different mirror map on each step (e.g.,
\citealp{duchi2011adaptive}). The version of this algorithm relevant in the context of
this paper starts from an arbitrary initialization $x_0 \in \cX$ and makes
updates of the form,
\begin{align} \label{eq:update}
x_{t+1} &=
	\argmin_{x \in \cX} \Lrset{g_t \cdot x + \thalf\norm{x-x_t}_{H_t^*}^2} ~.
\end{align}
This update is equivalent to the one in \cref{alg-step4} of \cref{alg:universal},
as shown in the appendix.
\begin{lemma}
\label{cor:regret-md}
For any $x^\st \in \cX$,  $g_1,\ldots,g_T \in \reals^d$ and $H_1,\ldots,H_T
\in \cone$, if $x_t$ are provided according to \cref{eq:update}, the
following bound holds,
\begin{align*}
\sum_{t=1}^T g_t \cdot (x_t-x^\star)
\le
	\half \sum_{t=1}^T \progr_t(x^\st)
	+ \half \sum_{t=1}^T \norm{g_t}_{H_t}^2
~ .
\end{align*}
\end{lemma}

For completeness, the proofs of both lemmas are given in \cref{tech:app}. We
now proceed with a short proof of the theorem.
\begin{proof}[Proof of \cref{thm:universal}]
From the convexity of $f_t$, it follows that
$f_t(x_t) - f_t(x^\st) \le g_t\cdot(x_t - x^*)$.  We thus get,
\begin{align*}
\sum_{t=1}^T f_t(x_t) - \sum_{t=1}^T f_t(x^\st)
\le \sum_{t=1}^T g_t\cdot(x_t - x^\st) .
\end{align*}
Hence, to obtain the claim from \cref{cor:regret-md} we need to show that
\begin{align*}
	\sum_{t=1}^T \norm{g_t}_{H_t}^2
		&\le G_T \bullet H_T + \Phi(H_T) - \Phi(H_0)~ .
\end{align*}
To this end, define functions $\psi_0,\psi_1,\ldots,\psi_T$ by setting $\psi_0(H) = G_0 \bullet H + \Phi(H)$, and
\begin{align*}
\psi_t(H) = g_t g_t\tr \bullet H ~
\end{align*}
for $t\ge 1$.
Then, by definition, $H_t$ is a minimizer of
$\sum_{s=0}^t \psi_s(H)$  over matrices $H \in \cH$.
\cref{lem:ftl-btl}
for the functions $\psi_0,\psi_1,\ldots,\psi_T$ now yields
$$
\sum_{t=1}^T \psi_t(H_t)
\le
\sum_{t=1}^T \psi_t(H_T) + \psi_0(H_T) - \psi_0(H_0)
~ .
$$
Expanding the expressions for the $\psi_t$, we get
\begin{align*}
\sum_{t=1}^T \norm{g_t}_{H_t}^2
&\le\ 
\sum_{t=1}^T \norm{g_t}_{H_T}^2 + \Phi(H_T) + G_0 \bullet H_T
  -\Phi(H_0) - G_0\bullet H_0\\
&= G_T \bullet H_T + \Phi(H_T) -\Phi(H_0) - G_0\bullet H_0\\
&\le\ 
G_T \bullet H_T + \Phi(H_T)  - \Phi(H_0)~ .&&\qedhere
\end{align*}
\end{proof}

\subsection{Spectral regularization}
As we show in the sequel, the potential $\Phi$ will often have the
form $\Phi(H) = -\trace(\phi(H))$, where $\phi : \reals^+ \mapsto
\reals$ is a (scalar) monotonically increasing function with a
positive first derivative.  We call this a {\em spectral potential}.
In this case $\nabla \Phi(H) = -\phi'(H)$, and further, if $\cH =
\cone$ then \cref{alg-step2} of the algorithm becomes
\begin{align} \label{eq:spectral}
H_t = (\phi')^{-1}(G_t).
\end{align}
Hence, the derivation of concrete algorithms from the general framework becomes extremely simple for spectral potentials and amounts to a simple transformation of the eigenvalues of the matrix $G_t$.
Furthermore, as we discuss below, spectral potentials make the
derivation of simplified diagonal (and scalar) versions of the
algorithms a straightforward task.

\subsection{Diagonal regularization}
To obtain a diagonal version of \cref{alg:universal}, i.e., a version in which
the maintained matrices $H_t$ are restricted to be diagonal, the only
modification required in the algorithm is to set $\cH$ to be the set of all
positive definite \emph{diagonal} matrices, denoted $\coneD$.  Specifically,
when $\Phi(H) = -\trace(\phi(H))$ is a spectral potential and $\cH = \coneD$,
then \cref{alg-step2} of the algorithm becomes
\begin{align} \label{eq:spectral-diag}
H_t
=
(\phi')^{-1}\Lr{ \!\diag(G_t) }
.
\end{align}
Indeed, for a diagonal $H$ we have
$G_t \bullet H - \Phi(H) = \diag(G_t) \bullet H - \Phi(H)$, and the
minimizer of the latter over all positive
definite matrices, according to \cref{eq:spectral}, is the matrix
$(\phi')^{-1}(\diag(G_t))$. Since the latter is a diagonal matrix, it is also
the minimizer of $G_t \bullet H - \Phi(H)$ over all \emph{diagonal} positive
definite matrices.
Consequently, diagonal versions of adaptive algorithms are obtained by
replacing the full matrix $G_t$ in \cref{alg:universal} with its
diagonal counterpart $\diag(G_t)$. 
$H_t = (\phi')^{-1}(\wt{G}_t)$ instead of $H_t = (\phi')^{-1}(G_t)$.  In
\cref{sec:adagrad-diag,sec:ons-diag} below, we spell out how this is
accomplished for the AdaGrad and ONS algorithms.

\subsection{Isotropic regularization}
To obtain the corresponding scalar versions of \cref{alg:universal} (namely,
analogous algorithms that only adaptively maintain a single scalar step-size),
we can modify the algorithm so that $H$ is optimized over the set $\coneI =
\set{s I : s>0}$ of all positive multiples of the identity matrix.
If we let $\Phi(H) = -\trace(\phi(H))$ be a spectral potential and let $\cH = \coneI$, then the update in \cref{alg-step2} of the algorithm is equivalent to
\begin{align} \label{eq:spectral-scalar}
H_t
=
(\phi')^{-1}\Lr{ \tfrac{1}{d}\trace(G_t) I }
.
\end{align}
To see this, note that for $H \in \coneI$ we have
$$G_t \bullet H + \Phi(H) = \tfrac{1}{d}\trace(G_t) I \bullet H + \Phi(H) ~.$$
Since the minimizer of the latter over all positive definite matrices is
$(\phi')^{-1}(\tfrac{1}{d}\trace(G_t) I) \in \coneI$, it is also the minimizer
over $\coneI$.%
\footnote{We note that we could have arrived at the same result by choosing the potential $\Phi(H) = \norm{H^{-1}}_2$ and minimizing over $\cH = \cone$.
However, notice that this is not a spectral potential and its analysis is more technically involved.
}

See \cref{sec:adagrad-scalar,sec:ons-scalar} on how scalar versions of the
AdaGrad and Online Newton Step algorithms are obtained by means of this
simple technique. 

\section{\ar $\Rightarrow$ AdaGrad} \label{sec:adagrad}
We now derive AdaGrad~\citep{duchi2011adaptive} from the \ar
meta-algorithm. We first describe how to obtain the full-matrix version of the
algorithm. In~\cref{sec:adagrad-diag} we provide the derivation of
AdaGrad's diagonal version.
Finally, in~\cref{sec:adagrad-scalar} we show that the
well-studied adaptive version of online gradient descent can be viewed, and
derived based on our framework, as a scalar version of AdaGrad.
The three versions employ a potential parameterized by
$\eta > 0$,
\begin{align} \label{eq:phi-adagrad}
\PhiAG(H) = \eta^2 \trace(H^{-1}) ~,
\end{align}
and differ by the domain $\cH$ of admissible matrices $H$.
Since $\PhiAG$ is a spectral potential, as we can rewrite,
$\PhiAG(H) = -\trace(\phi(H))$ for $\phi(x) = -\eta^2 x^{-1}$.
Simple calculus yields that $(\phi')^{-1}(y) = \eta y^{-1/2}$, which
in turn gives that
\begin{align} \label{eq:argmin-adagrad}
\argmin_{H \succ 0} \Lrset{ H \bullet G + \PhiAG(H) } = \eta G^{-1/2} ~ .
\end{align}

\subsection{Full-matrix AdaGrad}
AdaGrad employs the following update on each iteration,
\begin{align*}
x_{t+1} = \proj_\cX^{G_t^{1/2}} \! \Lr{ x_t - \eta G_t^{-1/2} g_t } ~,
\tag{AdaGrad}
\end{align*}
where $g_t = \nabla f_t(x_t)$, $G_t = \eps I + \sum_{s=1}^t g_s g_s\tr$ for
all $t\ge0$, and $\eta > 0$ is the step-size parameter. In the analysis
below, we only assume that the domain $\cX$ is bounded and its Euclidean
diameter is bounded by $\diam=\max_{x,x'\in\cX} \|x-x'\|$.

To obtain AdaGrad from \cref{alg:universal}, we choose the potential function
$\PhiAG$ over the domain $\cH = \cone$ and set $G_0 = \eps I$.  The values of
the parameters $\eta$ and $\eps$ are determined in the sequel.  According to
\cref{eq:argmin-adagrad}, the norm-regularization matrices used by
\cref{alg:universal} are indeed $H_t = \eta G_t^{-1/2}$, the same used by
AdaGrad.
Note that for projecting back to the domain $\cX$, we can use a
projection with respect to the norm $\norm{\cdot}_{G_t^{1/2}}$ rather than
$\norm{\cdot}_{H_t^\star}$. Since the two norms only differ by a scale, this
difference has no effect on the projection step.
We now invoke \cref{thm:universal} and bound the second term of the
bound in \cref{eq:universal},
\begin{align*}
\sum_{t=1}^T \progr_t(x^\st)
&=
\frac{1}{\eta} (x_1-x^\star)\tr G_1^{1/2} (x_1-x^\star)
+
\frac{1}{\eta} \sum_{t=2}^T (x_{t}-x^\star)\tr (G_t^{1/2}-G_{t-1}^{1/2}) (x_t-x^\star) ~.
\end{align*}
We bound the left term using
$v\tr M v \le \norm{M} \norm{v}^2 \le \trace(M) \norm{v}^2$ for a matrix $M \succeq 0$ and a vector~$v$. 
Setting $v = x_1 - x^\star$ and recalling
that the diameter of $\cX$ is bounded by $\diam$, we get
\begin{align*}
\frac{1}{\eta} (x_1-x^\star)\tr G_1^{1/2} (x_1-x^\star) & \le
	\frac{\diam^2}{\eta} \trace\lr{ G_1^{1/2} } ~ .
\end{align*}
To bound the right term above we can use the same technique. We need to show
though that $G_t^{1/2}-G_{t-1}^{1/2} \succeq 0$ for all $t$. Indeed, the
difference is PSD since $G_t \succeq G_{t-1}$ and $x \mapsto x^{1/2}$ is
operator monotone. We thus get,
\begin{align*}
(x_{t}-x^\star)\tr (G_t^{1/2}-G_{t-1}^{1/2}) (x_t-x^\star) & \le
\frac{\diam^2}{\eta} \trace\Lr{ G_t^{1/2}-G_{t-1}^{1/2} } ~ .
\end{align*}
Combining the two bounds we get,
\begin{align*}
\sum_{t=1}^T \progr_t(x^\st)
&\le
\frac{\diam^2}{\eta} \trace(G_1^{1/2}) +
\frac{\diam^2}{\eta} \sum_{t=2}^T \trace(G_t^{1/2}-G_{t-1}^{1/2})
= \frac{\diam^2}{\eta} \trace(G_{T}^{1/2}) ~.
\end{align*}
Since $G_T \bullet H_T = \eta \trace(G_T G_T^{-1/2}) = \eta \trace(G_T^{1/2})$,
together with $\Phi(H_T) = \eta \trace(G_T^{1/2})$ and a choice of
$\eta = \diam/\!\sqrt{2}$, we have
\begin{align*}
\sum_{t=1}^T f_t(x_t) - \sum_{t=1}^T f_t(x^\star)
\le
\lr{\eta + \frac{\diam^2}{2\eta}} \, \trace(G_{T}^{1/2})
= \sqrt{2}\,\diam \trace(G_{T}^{1/2}) ~,
\end{align*}
for any $x^\star \in \cX$. Note that $\eps$ can be taken arbitrarily small.

\subsection{Diagonal AdaGrad}
\label{sec:adagrad-diag}
\citet{duchi2011adaptive} presented a diagonal version of AdaGrad that uses
faster updates based on diagonal regularization matrices,
\begin{align*}
x_{t+1}
=
\proj_\cX^{\wt{G}_t^{1/2}} \Lr{ x_t - \eta \wt{G}_t^{-1/2} g_t } ~,
\tag{Diag AdaGrad}
\end{align*}
where $g_t = \nabla f_t(x_t)$ and $\wt{G}_t = \eps I + \diag(\sum_{s=1}^t g_s g_s\tr)$
for all $t$. Following~\citep{duchi2011adaptive}, in the analysis of
the diagonal algorithm we will assume a bound on the diameter of $\cX$ with
respect to the $\infty$-norm, which we denote by $\diam_\infty$.

In order to obtain the diagonal version of AdaGrad we choose the same
potential $\PhiAG$, but optimize over a domain $\cH$ restricted to
diagonal positive definite matrices.  From \cref{eq:spectral-diag} and
\cref{eq:argmin-adagrad}, the induced regularization matrices are
$\wt{H}_t = \eta \wt{G}_t^{-1/2}$, which recovers the diagonal version
of AdaGrad.

Invoking \cref{thm:universal} and repeating the arguments for
Full-matrix AdaGrad with
$\wt{H}_t$ replacing $H_t$, we obtain that
\begin{align*}
\sum_{t=1}^T \progr_t(x^\st)
&=
\frac{1}{\eta} (x_1-x^\star)\tr \wt{G}_1^{1/2} (x_1-x^\star)
+
\frac{1}{\eta} \sum_{t=2}^T (x_{t}-x^\star)\tr (\wt{G}_t^{1/2}-\wt{G}_{t-1}^{1/2}) (x_t-x^\star)
\\
&\le
\frac{\diam_\infty^2}{\eta} \trace(\wt{G}_1^{1/2}) +
\frac{\diam_\infty^2}{\eta} \sum_{t=2}^T \trace(\wt{G}_t^{1/2}-\wt{G}_{t-1}^{1/2})
=
\frac{\diam_\infty^2}{\eta} \trace(\wt{G}_{T}^{1/2}) ~ ,
\end{align*}
where the inequality uses the fact that for a diagonal positive semidefinite
matrix $A$, it holds that $v\tr A v \le \norm{v}_\infty^2 \trace(A)$ for any
vector $v$.
Furthermore, for the second sum in \cref{eq:universal} we have
$$G_T \bullet H_T =
	\eta \trace\left(G_T \bullet \wt{G}_T^{-1/2}\right) =
	\eta \trace(\wt{G}_T^{1/2}) ~ ,$$
where we used the elementary yet constructive fact that the support of
non-zeros of the product $G_T \bullet \wt{G}_T^{-1/2}$ is the same as
$\wt{G}_T \bullet \wt{G}_T^{-1/2}$.
Overall, with the choice of $\eta = \diam_\infty/\sqrt{2}$ we obtain the regret bound
\begin{align*}
\sum_{t=1}^T f_t(x_t) - \sum_{t=1}^T f_t(x^\star)
\le
\sqrt{2}\diam_\infty \trace\Lr{\wt{G}_{T}^{1/2}} ~ .
\end{align*}

\subsection{Isotropic AdaGrad: Adaptive Gradient Descent}
\label{sec:adagrad-scalar}
A classic adaptive version of the Online Gradient Descent (OGD) algorithm
uses standard projected gradient updates of the form
\begin{align*}
x_{t+1} = \proj_\cX \lr{ x_t - \eta_t g_t } \tag{Adaptive OGD} ~,
\end{align*}
with the decreasing step-size policy
$ \eta_t = \smash{c / \sqrt{ \tsum_{s \le t} \norm{g_s}^2 }} $
for an appropriate constant $c>0$.
For simplicity, we make the mild assumption that $\norm{g_1} > 0$ to avoid
degenerate cases.

We now show how the adaptive OGD algorithm is obtained from our framework as a
scalar version of AdaGrad.  To establish this, consider the potential $\PhiAG$
and fix the domain $\cH = \coneI = \set{s I : s > 0}$ to be the set of all positive
multiples of the identity matrix. Let us also set $G_0 = 0$. Recalling
\cref{eq:spectral-scalar}, the resulting regularization matrices used by
\cref{alg:universal} are $H_t = \eta (\tfrac{1}{d} \trace(G_t))^{-1/2} \Id$.
By setting $\eta = c/\sqrt{d}$ we get,
\begin{align*}
H_t g_t = \frac{c}{ \sqrt{ \tsum_{s=1}^t \norm{g_s}^2 } } g_t = \eta_t g_t ~ ,
\end{align*}
recovering the adaptive OGD algorithm.  In order to obtain a regret bound for the Isotropic AdaGrad algorithm, we can
apply \cref{thm:universal} and repeat the arguments for Full-matrix
AdaGrad. First, we have
\begin{align*}
\sum_{t=1}^T
\progr_t(x^\st) &=
\frac{1}{c} \sqrt{\trace(G_1)} \norm{x_1-x^\star}^2
+
\frac{1}{c}
\sum_{t=2}^T
	\LR{ \sqrt{\trace(G_t)}-\sqrt{\trace(G_{t-1})} } \norm{x_t-x^\star}^2
\le
\frac{\diam^2}{c} \sqrt{\trace(G_T)} ~ .
\end{align*}
We also use the following,
\begin{align*}
G_T\bullet H_T &+ \PhiAG(H_T) - \PhiAG(H_0)
\\
&=
c \trace(G_T)^{-1/2}\trace(G_T) + c \trace(G_T)^{1/2} - c \trace(G_0)^{1/2}\\
&\le
2 c \sqrt{\trace(G_T)}
.
\end{align*}
Overall, with the choice of $c = \diam/\sqrt{2}$ we obtain the regret bound
\begin{align*}
\sum_{t=1}^T f_t(x_t) - \sum_{t=1}^T f_t(x^\star)
\le
\diam \sqrt{2\trace(G_{T})}
=
\sqrt{2}\diam \; \lr{\sum_{t=1}^T \norm{g_t}^2}^{1/2} ~ .
\end{align*}

\subsection{A $\bm{p}$-norm extension of AdaGrad}
We conclude this section with a simple spectral extension of AdaGrad
which regularizes according to the $p$-norm of the spectral coefficients.
To do so, let us choose
$$\Phi_p(H) = \frac{\eta^{p+1}}{p} \trace\lr{H^{-p}} ~.$$
We then have,
$\Phi_p(H) = -\trace(\phi(H))$ where $\phi(x) = -(\eta^{p+1}/p) x^{-p}$,
and therefore $\Phi_p(H)$ is a spectral potential. Elementary calculus yields
$(\phi')^{-1}(y) = \eta y^{-1/(p+1)}$, which in turn gives
\begin{align} \label{eq:argmin-pnorm}
\argmin_{H \succ 0} \Lrset{ H \bullet G + \Phi_p(H) } =
	\eta G^{-1/(p+1)} ~~ .
\end{align}
We can use~\cref{thm:universal} as before to obtain the following regret
bound,
\begin{align*}
\sum_{t=1}^T f_t(x_t) - \sum_{t=1}^T f_t(x^\star)
\le
\frac{1}{2\eta} \diam^2\trace\Lr{G_{T}^{1/(p+1)}} +
	\eta\tfrac{p+1}{2p} \trace\Lr{G_T^{p/(p+1)}} ~ .
\end{align*}
We now set
$$\eta = \diam \sqrt{ \frac{p}{p+1}
  \trace\Lr{G_{T}^{1/(p+1)}} / \trace\Lr{G_T^{p/(p+1)}}} ~,$$
and obtain
\begin{align} \label{pbound:eqn}
\sum_{t=1}^T f_t(x_t) - \sum_{t=1}^T f_t(x^\star)
\le
\diam\sqrt{\tfrac{p+1}{p}
	\trace\Lr{G_T^{1/(p+1)}}
	\trace\Lr{G_T^{p/(p+1)}}} ~ .
\end{align}
Setting $p=1$ yields the AdaGrad update with the same regret bound
obtained above. Moreover, the choice of $p=1$ provides
the best regret bound among all choices for $p$. To see
that, let us denote the eigenvalues of $G_T$ by $\lambda_1,\ldots,\lambda_d \ge 0$.
The product of the two traces in \cref{pbound:eqn} amounts to
$$
\trace\Lr{G_{T}^{1/(p+1)}} \trace\Lr{G_T^{p/(p+1)}}
=
\LR{\sum \lambda_i^{1/(p+1)}} \LR{\sum_i \lambda_i^{p/(p+1)}}
\ge 
\LR{\sum \lambda_i^{1/2}}^2
=
\trace\Lr{G_{T}^{1/2}}^2
~,$$
by the Cauchy-Schwarz inequality.%
\footnote{We note, however, that the worst case analysis does not necessarily transfer to actual performance on real problems, and choosing a $p$-norm regularization with $p\neq 1$ may prove useful in practice.}

\section{\ar $\Rightarrow$ Online Newton Step}
\label{sec:ons}
We now show how to derive the Online Newton Step (ONS) algorithm
of~\citep{hazan2007logarithmic} from \cref{alg:universal}.
The ONS update takes the following form,
\begin{align*}
x_{t+1} =
	\proj_\cX^{G_t} \Lr{ x_t - \eta G_t^{-1} g_t } ~ ,
\tag{ONS}
\end{align*}
where as before $g_t = \nabla f_t(x_t)$ and
	$G_t = \eps I + \sum_{s=1}^t g_s g_s\tr$ for all $t\ge 0$.
Here again $\eta$ is a fixed step-size.
Throughout this section, we assume that the $f_t$ are $\gamma$-Lipschitz,
namely, $|f_t(x) - f_t(y)| \le \gamma \norm{x - y}$ for $x,y\in\cX$ and
the domain's diameter, $\max_{x,x'\in\cX} \norm{x-x'}\le\diam$, both with
respect to the Euclidean norm.

\subsection{Full-matrix ONS}
Let us first describe how the full-matrix version of ONS is derived
through a specific choice for potential, $\PhiONS(H)$. We assume that
the cost functions $f_1,\ldots,f_T$ are $\beta$-exp-concave over the
domain $\cX \subseteq \reals^d$, with the following minor abuse of
terminology. Concretely, we assume that for all $t\ge1$ and $x,y \in \cX$,
\begin{align} \label{eq:expcon}
f_t(x) - f_t(y) \le
	\nabla f_t(x) \cdot (x-y) - \frac{\beta}{2}
		\Lr{ \nabla f_t(x) \cdot (x-y) }^2 ~ .
\end{align}
We refer the reader to~\citep{hazan2007logarithmic,hazan2016introduction} for
further background on exp-concavity and its precise definition.%

To obtain the ONS update, we use the potential,
\begin{align*}
\PhiONS(H) = -\frac{1}{\beta} \log\det{H}
~ ,
\end{align*}
over $\cH = \cone$ and choose a fixed $G_{0} = \eps I$ for some $\eps > 0$.
Since $\PhiONS(H)$ is equal to $-\trace(\phi(H))$
where $\phi(x) = \beta^{-1} \log{x}$, $\PhiONS$ is a spectral potential.
We get that $$(\phi')^{-1}(z) = (\beta z)^{-1} ~,$$ and thus
\begin{align} \label{eq:argmin-ons}
\argmin_{H \succ 0} \Lrset{ H \bullet G + \PhiONS(H) }
	= \frac{1}{\beta} G^{-1} ~ .
\end{align}
Therefore, the pre-conditioning matrices induced by the potential $\PhiONS$ in
\cref{alg:universal} are $H_t = (1/\beta)G_t^{-1}$, which gives rise to the
update rule used by ONS.

We proceed to analyze the regret of ONS by means of \cref{thm:universal}.
Bounding from above the terms of the right-hand side of~\cref{eq:universal},
we have that
\begin{align*}
\sum_{t=1}^T \progr_t(x^\st)
&=
\beta (x_1-x^\star)\tr G_1 (x_1-x^\star)
+
\beta \sum_{t=2}^T (x_{t}-x^\star)\tr (G_t-G_{t-1}) (x_t-x^\star)
\\
&=
\beta (x_1-x^\star)\tr G_0 (x_1-x^\star)
+
\beta \sum_{t=1}^T (x_{t}-x^\star)\tr (G_t-G_{t-1}) (x_t-x^\star) \\
&\le
\eps \beta \diam^2
+
\sum_{t=1}^T \beta \Lr{g_t \cdot (x_{t}-x^\star)}^2
~ ,
\end{align*}
since $G_t - G_{t-1} = g_t g_t\tr$.
In addition we have,
$G_T \bullet H_T = \frac{1}{\beta} \trace\Lr{ G_T^{-1} G_T } = \frac{d}{\beta}$.
Let us denote the eigenvalues of
$G_T - G_0 = \sum_{t=1}^T g_t g_t\tr$ by
$\lambda_1,\ldots,\lambda_d \ge 0$.
Then, the eigenvalues of $G_T$ are
$\lambda_1+\eps,\ldots,\lambda_d+\eps$ and $\PhiONS(H_T)$ is
bounded as follows,
\begin{align*}
\PhiONS(H_T) - \PhiONS(H_0) 
&= \PhiONS\left(\frac{1}{\beta}
  G_T^{-1}\right) - \PhiONS\left(\frac{1}{\beta}
  G_0^{-1}\right)\\
&=
\frac{1}{\beta}\log\frac{\det{G_T}}{\det{G_0}}
=
\frac{1}{\beta} \sum_{i=1}^d \log\lr{1 + \frac{\lambda_i}{\eps}}
\\
&\le
\frac{d}{\beta} \log\lr{1+\frac{\gamma^2 T}{\eps}} ~ .
\end{align*}
The inequality above stems from to the fact that the eigenvalues
$\lambda_i$ of $G$ are upper bounded by $\gamma^2 T$, since
$\norm{g_t}^2 \le \gamma^2$ for all $t$ due to the $\gamma$-Lipschitz
assumption. In order to put everything together let us define
$\tilde{f}_t(x) = g_t \cdot x$ and apply Theorem~\cref{eq:universal}
to $\tilde{f}_t$, we obtain
\begin{align*}
\sum_{t=1}^T \left(f_t(x_t) - f_t(x^\star)\right) & \le
\sum_{t=1}^T \Lr{ g_t \cdot (x_t-x^\star) - \tfrac{\beta}{2}
	\Lr{ g_t \cdot (x_t-x^\star) }^2 } \tag{From \cref{eq:expcon}}\\
&=
\sum_{t=1}^T \Lr{ \tilde{f}_t(x_t) - \tilde{f}_t(x^\star) } -
\sum_{t=1}^T \tfrac{\beta}{2} \Lr{ g_t \cdot (x_t-x^\star) }^2
	\tag{Definition of $\tilde{f}_t$}\\
&\le \half \sum_{t=1}^T \Lr{\progr_t(x^\st) -
		\beta \Lr{ g_t \cdot (x_t-x^\star) }^2} \\
& \hspace{1cm} +
	\half \Lr{G_T \bullet H_T + \PhiONS(H_T) - \PhiONS(H_0)}
	\tag{\cref{thm:universal}} \\
&\le \frac{\eps \beta \diam^2}{2} + \frac{d}{2\beta}
	\Lr{1 + \log\lr{1+{\gamma^2 T}/\eps}} ~ .
\end{align*}
Taking $\eps = d/(\beta^2 \diam^2)$ gives the regret bound
\begin{align*}
\sum_{t=1}^T f_t(x_t) - \sum_{t=1}^T f_t(x^\star)
\le
\frac{d}{\beta}\lr{1 + \log\lr{\frac{(\beta \gamma \diam)^2 T}{d}}} ~ .
\end{align*}

\subsection{Diagonal ONS}
\label{sec:ons-diag}
We now turn to develop a diagonal version of ONS.
We will show that the diagonal version guarantees $O(d\log{T})$ regret under a coordinate-wise analogue of the exp-concavity property.
Formally (and slightly abusing terminology again), we say that the function $f_t$ is $\beta$-coordinate-wise exp-concave over a domain $\cX$ if the following holds for all points $x,y \in \cX$:
\begin{align} \label{eq:expcon-coo}
f_t(x) - f_t(y)
\le
\nabla f_t(x) \cdot (x-y)
- \frac{\beta}{2} (\nabla f_t(x))^2 \cdot (x-y)^2 
.
\end{align}
Here, we use the shorthand $u^2 = (u_1^2,\ldots,u_d^2)$ for vectors $u \in
\reals^d$.  In general, coordinate-wise exp-concavity is not comparable to
standard exp-concavity. Namely, if $f_t$ satisfies \cref{eq:expcon} then it
may not satisfy \cref{eq:expcon-coo} and vice versa.  Nonetheless, under
either property one can obtain a $O(d\log{T})$ regret using a corresponding
algorithm.

To obtain the diagonal ONS algorithm, we use the same potential $\PhiONS$ and
restrict $\cH$ to be the set of all {\em diagonal} positive definite matrices.
Then, letting $\wt{G}_t = \diag(G_t)$, \cref{eq:spectral-diag} implies that
the induced regularization matrices in \cref{alg:universal} are of the form
$\wt{H}_t = (1/\beta)\wt{G}_t^{-1}$, which gives rise to the following update
rule,
\begin{align*}
x_{t+1}
=
\proj_\cX^{\wt{G}_t} \Lr{ x_t - \tfrac{1}{\beta} \wt{G}_t^{-1} g_t }
\tag{Diagonal ONS}
~ \mbox{ where } ~ \wt{G}_t = \eps I + \diag\LR{\sum_{s=1}^t g_s g_s\tr} ~.
\end{align*}
We repeat the derivation for Full-matrix ONS with diagonal matrices
$$H_t = (1/\beta)\wt{G}_t^{-1} = (1/\beta) \diag(G_t)^{-1} ~,$$ and get,
\begin{align*}
\sum_{t=1}^T \progr_t(x^\st)
&=
\beta (x_1-x^\star)\tr \diag(G_1)\, (x_1-x^\star) +
\beta \sum_{t=2}^T (x_{t}-x^\star)\tr \diag(G_t-G_{t-1})\, (x_t-x^\star) \\
&\le \eps \beta \diam^2 +
	\beta \sum_{t=1}^T \sum_{i=1}^d g_{t,i}^2 (x_{t,i}-x^\star_i)^2 ~.
\end{align*}
The bound on $G_T\bullet H_T$ amounts to,
\begin{align*}
G_T\bullet H_T &= \frac{1}{\beta}
	\trace\Lr{ \diag(G_T)^{-1} G_T } = \frac{d}{\beta} ~.
\end{align*}
Finally, we bound the difference in the potential of $H_0$ and $H_T$ as
follows,
\begin{align*}
\PhiONS(H_T) &- \PhiONS(H_0) =
\frac{1}{\beta} \log{\frac{\det{\diag(G_T)}}{\det{\diag(G_0)}}}
\le
\frac{d}{\beta} \log\lr{1+\frac{\gamma^2 T}{\eps}}
~ .
\end{align*}
In summary, using \cref{thm:universal} we obtain the bound,
\begin{align*}
\sum_{t=1}^T \sum_{i=1}^d \LR{ g_{t,i} (x_{t,i}-x^\star_i) -
	\tfrac{\beta}{2} g_{t,i}^2 (x_{t,i}-x^\star_i)^2 } \le
\frac{\eps \beta \diam^2}{2} +
	\frac{d}{2\beta}\lr{1 + \log\lr{1+\frac{\gamma^2 T}{\eps}}} ~ .
\end{align*}
For $\beta$-coordinate-wise exp-concave functions, as defined above,
the left-hand side upper-bounds the regret.
Taking $\eps = d/(\beta^2 \diam^2)$ gives, for any
$x^\star \in \cX$, the regret bound
\begin{align*}
\sum_{t=1}^T f_t(x_t) - \sum_{t=1}^T f_t(x^\star)
=
\frac{d}{\beta}\lr{1 + \log\frac{(\beta \gamma \diam)^2 T}{d} }
.
\end{align*}

\subsection{Isotropic ONS: Strongly-convex OGD}
\label{sec:ons-scalar}
Finally, we derive a scalar version of the ONS algorithm and show that it
yields an adaptive version of the Online Gradient
Descent (OGD) algorithm in the strongly convex
case~\citep{hazan2007logarithmic, shalev2009mind}.
This version performs the following update,
\begin{align*}
x_{t+1}
=
\proj_\cX \lr{ x_t - \eta_t g_t }
\tag{Strongly Convex OGD}
\end{align*}
with $\eta_t = \Theta(1/t)$.
If the functions $f_t$ are $\alpha$-strongly convex, namely they satisfy
\begin{align} \label{eq:strconv}
f_t(x) - f_t(y)
\le
\nabla f_t(x) \cdot (x-y) - \tfrac{\alpha}{2} \norm{x-y}^2
,
\qquad\qquad
\forall ~ x,y \in \cX
,
\end{align}
then such an algorithm can achieve $O(\log{T})$ regret
\citep{hazan2007logarithmic,shalev2009mind}. To obtain and analyze this
version, we yet again use the potential $\PhiONS$ and restrict the domain
$\cH = \coneI = \set{s I : s > 0}$.  In addition, we set $G_0 = (\eps/d) I$ where
$\eps$ is determined below.
From~\cref{eq:spectral-scalar}, we see that \cref{alg:universal}
uses the regularization matrices $H_t = \frac{d}{\beta\trace(G_t)} I$.
Since $\trace(G_t) = \eps+\sum_{s=1}^t \norm{g_s}^2$, the resulting update
is of the same form as OGD with an adaptive learning rate of,
\begin{align*}
\eta_t = \frac{d/\beta}{\eps+\sum_{s=1}^t \norm{g_s}^2} ~ .
\end{align*}
It is evident that $\eta_t$ roughly decays at a rate of $1/t$. 

Applying \cref{thm:universal} with
$H_t = \frac{d}{\beta\trace(G_t)} I$ and repeating the same calculations
above, we have
\begin{align*}
\sum_{t=1}^T \progr_t(x^\st)
&=
\frac{\beta}{d} \trace(G_0) \norm{x_1-x^\star}^2 +
\frac{\beta}{d} \sum_{t=1}^T \norm{g_t}^2 \, \norm{x_t-x^\star}^2 \\
&\le
\frac{\eps\beta}{d} \norm{x_1-x^\star}^2 + \frac{\beta \gamma^2}{d}
	\sum_{t=1}^T \norm{x_t-x^\star}^2 ; \\
G_T\bullet H_T &+ \PhiONS(H_T) - \PhiONS(H_0)
=
\frac{d}{\beta}
+ \frac{1}{\beta} \log\frac{(\trace(G_T))^d}{(\trace(G_0))^d}
\\
&\le
\frac{d}{\beta} \lr{ 1 + \log\lr{1 + \frac{\gamma^2 T}{\eps}} } ~
.
\end{align*}
Overall, we obtain
\begin{align*}
\sum_{t=1}^T
	\LR{ g_t \cdot (x_t-x^\star) - \frac{\beta\gamma^2}{2d} \norm{x_t-x^\star}^2 }
\le
\frac{\eps\beta}{2d} \norm{x_1-x^\star}^2 + \frac{d}{2\beta} \lr{ 1 +
  \log\lr{1 + \frac{\gamma^2 T}{\eps} }} ~ ,
\end{align*}
Setting $\beta = \alpha d/\gamma^2$ and recalling \cref{eq:strconv}, we see
that the left-hand side bounds the regret for $\alpha$-strongly-convex
functions, and we thus get
\begin{align*}
\sum_{t=1}^T f_t(x_t) - \sum_{t=1}^T f_t(x^\star)
&\le
\sum_{t=1}^T \lr{ g_t \cdot (x_t-x^\star) - \frac{\alpha}{2} \norm{x_t-x^\star}^2 }\\
&\le
\frac{\eps \alpha}{2\gamma^2} \norm{x_1-x^\star}^2 +
  \frac{\gamma^2}{2\alpha} \lr{ 1 + \log\lr{1 + \frac{\gamma^2 T}{\eps} }} ~
.
\end{align*}
To bound $\norm{x_1-x^\star}$ let $\tilde{x}$ be the
minimizer of $f(x) = \frac{1}{T} \sum_{t=1}^T f_t(x)$, which is a
$\gamma$-Lipschitz and $\alpha$-strongly convex function. Thus,
$$
0 
\le 
f(x) - f(\tilde{x})
\le 
\nabla f(x) \cdot (x-\tilde{x}) - \tfrac{\alpha}{2}\norm{x-\tilde{x}}^2
\le 
\gamma \norm{x-\tilde{x}} - \tfrac{\alpha}{2}\norm{x-\tilde{x}}^2 ~,
$$
thus $\norm{x - \tilde{x}} \le \frac{2\gamma}{\alpha}$ and
$\norm{x-x^\star} 
\le
\norm{x - \tilde{x}} + \norm{\tilde{x}-x^\star} 
\le
\frac{4\gamma}{\alpha}$.
A choice of $\eps = \gamma^2$ gives us the regret bound
\begin{align*}
\sum_{t=1}^T f_t(x_t) - \sum_{t=1}^T f_t(x^\star)
\le \frac{\gamma^2}{\alpha} \lr{8 + \log{T} } ~ .
\end{align*}

\section{Brief Discussion} \label{discussion:sec}
While our focus in this paper was on the derivation of AdaGrad and ONS as
special cases of \ar, our apparatus can be used to derive new adaptive
regularization algorithms. In addition to the $p$-norm regularization above,
it is also seamless to derive block-diagonal variants of \ar, which
structurally interpolate between the full and diagonal versions.  A more
challenging task is to generalize \ar to incorporate other, more complex
structures of matrices that can efficiently capture intricate dependencies
between different parameters. Another possible extension left to explore is
allowing time-varying potentials in \ar. We plan to study additional
potential functions and examine their empirical properties in future work.

\subsection*{Acknowledgments}
We would like to thank Elad Hazan and Roi Livni for numerous
stimulating discussions in the early stages of this research.  We also thank
Roy Frosting and Kunal Talwar for helpful comments and suggestions.

\bibliographystyle{abbrvnat}
\bibliography{adareg}

\appendix

\section{Online optimization mini-toolchest}
\label{tech:app}

\paragraph{The FTL-BTL lemma.}
We give for completeness a proof of the FTL-BTL Lemma.
Note that $\psi_0$ can be viewed as a regularization term.
\begin{lemma}[FTL-BTL Lemma] \label{sec:ftl-btl}
Let $\psi_0,\ldots,\psi_T : \cX \mapsto \reals$ be an arbitrary sequence of
functions defined over a domain $\cX$.
For each $t \ge 0$, let $x_t \in \argmin_{x \in \cX} \sum_{s=0}^t \psi_s(x)$.
Then, the following inequality holds for $T\ge 1$,
\begin{align*}
\sum_{t=1}^T \psi_t(x_t) \le
	\sum_{t=1}^T \psi_t(x_T) + (\psi_0(x_T) - \psi_0(x_0)) ~ .
\end{align*}
\end{lemma}
\begin{proof}
We rewrite above equation as,
$ \sum_{t=0}^T \psi_t(x_t) \le \sum_{t=0}^T \psi_t(x_T)$.
This inequality is true for $T = 0$.  Assume it holds true for
$T-1$, then
\begin{align*}
\sum_{t=0}^{T-1} \psi_t(x_t)
\le
\sum_{t=0}^{T-1} \psi_t(x_{T-1})
\le
\sum_{t=0}^{T-1} \psi_t(x_{T}) ~
,
\end{align*}
where the second inequality follows since
$x_{T-1} \in \argmin_{x \in \cX} \sum_{s=0}^{T-1} \psi_s(x)$.
Adding $\psi_T(x_{T})$ to both sides proves the result for $T$,
completing the induction.
\end{proof}

\paragraph{Online Mirror descent with time-varying norms.}
The abstract version of online mirror descent (originated from \citealp{KivinenWa97}), which is an online version
of the classic mirror descent method \citep{nemirovskii1983problem}, employs a Bregman
divergence to construct its update. As this paper is focused solely
on divergences of the form $\|\cdot\|_H$ we present a simplified
analysis of mirror descent for this specific setting. In this confined
form, mirror descent sets the next iterate $x_{t+1}$ as follows,
\begin{align} \label{eq:update2}
x_{t+1} &=
	\argmin_{x \in \cX} \Lrset{g_t \cdot x + \thalf\|x-x_t\|_{H_t^*}^2} ~ ,
\end{align}
where $g_t$ is an arbitrary vector (usually, one takes $g_t = \nabla
f_t(x_t)$, but for analysis below this need not be the case). This
update is equivalent to \cref{alg-step4} of \cref{alg:universal} as follows:
\begin{align*}
\Pi_{\cX}^{H_t^*}\Lr{ x_t - H_t g_t } 
&= \argmin_{x\in \cX} \Lrset{ \norm{x -
  (x_t - H_tg_t)}^2_{H_t^*} } \\
&= \argmin_{x\in \cX} \Lrset{ \norm{x - x_t}_{H_t^*}^2 + 2 (x-x_t)\tr
  H_t^{-1} H_tg_t + \norm{H_tg_t}_{H_t^*}^2 } \\
&= \argmin_{x\in \cX} \Lrset{ \thalf\norm{x - x_t}_{H_t^*}^2 + x\cdot g_t } ,
\end{align*}
where we eliminated terms independent of $x$.

The regret bound of mirror descent is provided
by the following lemma.
\begin{lemma} \label{lem:regret-md}
For any $x\in\cX$, $g_1,\ldots,g_T \in \reals^d$ and
$H_1,\ldots,H_T \in \cone$, if $x_t$ are given by \cref{eq:update2},
the following holds:
\begin{align*}
\sum_{t=1}^T g_t \cdot (x_t - x)
\le
\half \sum_{t=1}^T \lr{
\norm{x_t-x}_{H_t^*}^2 - \norm{x_{t+1}-x}_{H_t^*}^2}
+ \half \sum_{t=1}^T \norm{g_t}_{H_t}^2
.
\end{align*}
\end{lemma}

\begin{proof}
First-order optimality conditions imply that the minimum $y$ of a convex
function $\psi$ subject to domain constraints $x\in\cX$, then
$\nabla\psi(y)\cdot(y-x) \le 0$ for all $x\in\cX$. Taking as $\psi$ the
objective minimized by mirror descent on round $t$, then
$\nabla\psi(x_{t+1}) = g_t + (x_{t+1} - x_t)\tr H_t^{-1} $ and thus for $x\in\cX$,
$$
(g_t  +  (x_{t+1} - x_{t})\tr H_t^{-1}) \cdot (x_{t+1} - x) \le 0 .
$$
Re-arranging terms, we get:
\begin{align} \label{eq:gt1}
\begin{split}
g_t\cdot (x_{t+1} - x) &\le (x_{t} - x_{t + 1})\tr H_t^{-1}  (x_{t+1} - x)\\
&= \thalf \norm{x_t - x}_{H_t^*}^2  - \thalf \norm{x_{t+1} -
  x}_{H_t^*}^2 - \thalf \norm{x_t - x_{t+1}}_{H_t^*}^2 ,
\end{split}
\end{align}
where the equality can be verified by expanding both sides.
On the other hand, using H\"older's inequality and the fact that $ab \le \half (a^2 + b^2)$ we obtain
\begin{align} \label{eq:gt2}
	g_t \cdot (x_t - x_{t+1})
\le
	\norm{g_t}_{H_t} \cdot \norm{x_{t+1} - x_t}_{H_t^*}
\le
	\thalf \norm{g_t}_{H_t}^2 + \thalf \norm{x_{t+1} - x_t}_{H_t^*}^2 .
\end{align}
Adding \cref{eq:gt1,eq:gt2}, we have
\begin{align*}
	g_t \cdot (x_t - x)
\le
	\thalf \norm{g_t}_{H_t}^2 +\thalf \norm{x_t - x}_{H_t^*}^2  - \thalf \norm{x_{t+1} - x}_{H_t^*}^2 .
\end{align*}
Summing over all $t$, we have the required inequality.
\end{proof}

\end{document}